\documentclass[nohyperref]{article}

\usepackage{microtype}
\usepackage{graphicx}
\usepackage{subfigure}
\usepackage{booktabs} 

\usepackage{hyperref}

 \usepackage[accepted]{icml2022}

\usepackage{amsmath}
\usepackage{amssymb}
\usepackage{mathtools}
\usepackage{amsthm}
\usepackage{tabularx}
\usepackage{multirow}
\usepackage{color}

\usepackage[capitalize,noabbrev]{cleveref}

\theoremstyle{plain}
\newtheorem{theorem}{Theorem}[section]

\newtheorem{fact}[theorem]{Fact}
\theoremstyle{definition}
\newtheorem{definition}[theorem]{Definition}

\theoremstyle{remark}
\newtheorem{remark}[theorem]{Remark}
\newtheorem{example}[theorem]{Example}

\usepackage[textsize=tiny]{todonotes}

\def\mcl#1{\mathcal{#1}}
\def\blacket#1{\left\langle #1\right\rangle}

\def\opn{\operatorname}

\def\alg{\mcl{A}}
\def\modu{\mcl{M}}

\def\red#1{\textcolor{black}{#1}}
\def\mb#1{\mathbf{#1}}
\def\bs#1{\boldsymbol{#1}}

\icmltitlerunning{C*-algebra Net}

\begin{document}

\twocolumn[
\icmltitle{$C^*$-algebra Net: A New Approach Generalizing \\ Neural Network Parameters to $C^*$-algebra}



\icmlsetsymbol{equal}{*}

\begin{icmlauthorlist}
\icmlauthor{Yuka Hashimoto}{ntt}
\icmlauthor{Zhao Wang}{ntt,wsd}
\icmlauthor{Tomoko Matsui}{ims}
\end{icmlauthorlist}

\icmlaffiliation{ntt}{NTT Network Service Systems Laboratories, NTT Corporation, Tokyo, Japan}
\icmlaffiliation{wsd}{Institute for Disaster Response Robotics, Waseda University, Tokyo, Japan}
\icmlaffiliation{ims}{Department of Statistical Modeling, the Institute of Statistical Mathematics, Tokyo, Japan}

\icmlcorrespondingauthor{Yuka Hashimoto}{yuka.hashimoto.rw@hco.ntt.co.jp}

\icmlkeywords{$C^*$-algebra, Hilbert $C^*$-module, Density estimation, Few-shot learning}

\vskip 0.3in
]



\printAffiliationsAndNotice{}  

\begin{abstract}
We propose a new framework that generalizes the parameters of neural network models to $C^*$-algebra-valued ones. $C^*$-algebra is a generalization of the space of complex numbers. A typical example is the space of continuous functions on a compact space. This generalization enables us to combine multiple models continuously and use tools for functions such as regression and integration. Consequently, we can learn features of data efficiently and adapt the models to problems continuously. We apply our framework to practical problems such as density estimation and few-shot learning and show that our framework enables us to learn features of data even with a limited number of samples. Our new framework highlights the potential possibility of applying the theory of $C^*$-algebra to general neural network models.
\end{abstract}

\section{Introduction}\label{sec:intro}
Continuation of neural network models has been discussed and successfully applied to practical problems and theoretical analyses.
\citet{chen18} proposed to regard the transformation of input variables as a continuous dynamical system, which is called the neural ODE.
While the neural ODE focus on the continuation of layers (vertical direction), continuation of units (horizontal direction) is also discussed.  
By regarding the action of a weight matrix to variables as integrations and using an integral representation of the model, we can use the theory of harmonic analysis to analyze the model theoretically~\cite{candes99,sonoda17,sonoda21}.
In these frameworks, we regard the parameters of a classical model being obtained by a discretization of functions, which enables us to use tools for functions such as derivative and integration for practical applications and theoretical analyses of the model.  
 
In this paper, we propose a {brand} new framework that generalizes the parameters (weights) of models. 
Unlike previous works regarding the continuation of neural network models, we do not consider the continuation of the parameters of {\em a single} classical model to functions.
Instead, we generalize each $\mathbb{R}$-valued parameter to a $C^*$-algebra-valued one, which corresponds to an aggregation of {\em multiple} classical models and a continuation of the parameters of the multiple models to functions.
$C^*$-algebra is a generalization of the space of complex numbers.
Typical examples are the space of continuous functions on a compact space, the space of $L^{\infty}$ functions on a $\sigma$-finite measure space, and the space of bounded linear operators on a Hilbert space.
In this paper, we focus on the $C^*$-algebra of the space of continuous functions on a compact space.
Practically, we can use tools for functions such as regression and integration to learn multiple models efficiently and to adapt the models to problems continuously.
Theoretically, our new framework highlights the potential possibility of applying the theory of $C^*$-algebra to general neural network models.
Fig.~\ref{fig:overview} shows an overview of our framework schematically.
The application of $C^*$-algebra to data analysis is discussed by~\citet{hashimoto21}.
We focus on the practical applications to neural network in this paper.
{To the best of our knowledge, this is the first paper that applies the theory of $C^*$-algebra to neural network models.}

{Our contributions are as follows:\vspace{.2cm}\\
$\bullet$ We propose a generic framework of neural network with $C^*$-algebra-valued parameters.\\
$\bullet$ We propose a gradient descent method to learn the model on $C^*$-algebra. \\
$\bullet$ We apply our framework to practical problems such as density estimation and few-shot learning.\vspace{.2cm}\\}
Regarding the second one, since $C^*$-algebra admits a generalization of Hilbert space, which is called Hilbert $C^*$-module, we can generalize the classical gradient descent method on the Hilbert space associated with the parameters of a model to that on the Hilbert $C^*$-module associated with the $C^*$-algebra-valued parameters.

The remainder of this paper is organized as follows.
In Section~\ref{sec:background}, we review mathematical notions related to $C^*$-algebra and Hilbert $C^*$-module.
In Section~\ref{sec:cstar_net}, we propose a model with $C^*$-algebra-valued parameters.
Then, in Section~\ref{sec:application}, we discuss practical applications and show numerical results.
We conclude the paper in Section~\ref{sec:conclusion}. 
The source code of this paper is available at \url{https://www.rd.ntt/e/ns/qos/person/hashimoto/code_c_star_net.zip}.

\paragraph{Notations}
Bold letters denote $\mathbb{R}$-valued objects or maps from $\mathbb{R}^{d_1}$ to $\mathbb{R}^{d_2}$ for some $d_1,d_2\in\mathbb{N}$.
Italic letters denote $\alg$-valued objects or maps from $\alg^{d_1}$ to $\alg^{d_2}$.
\color{black}
\begin{figure}[t]
\centering
 \includegraphics[scale=0.25]{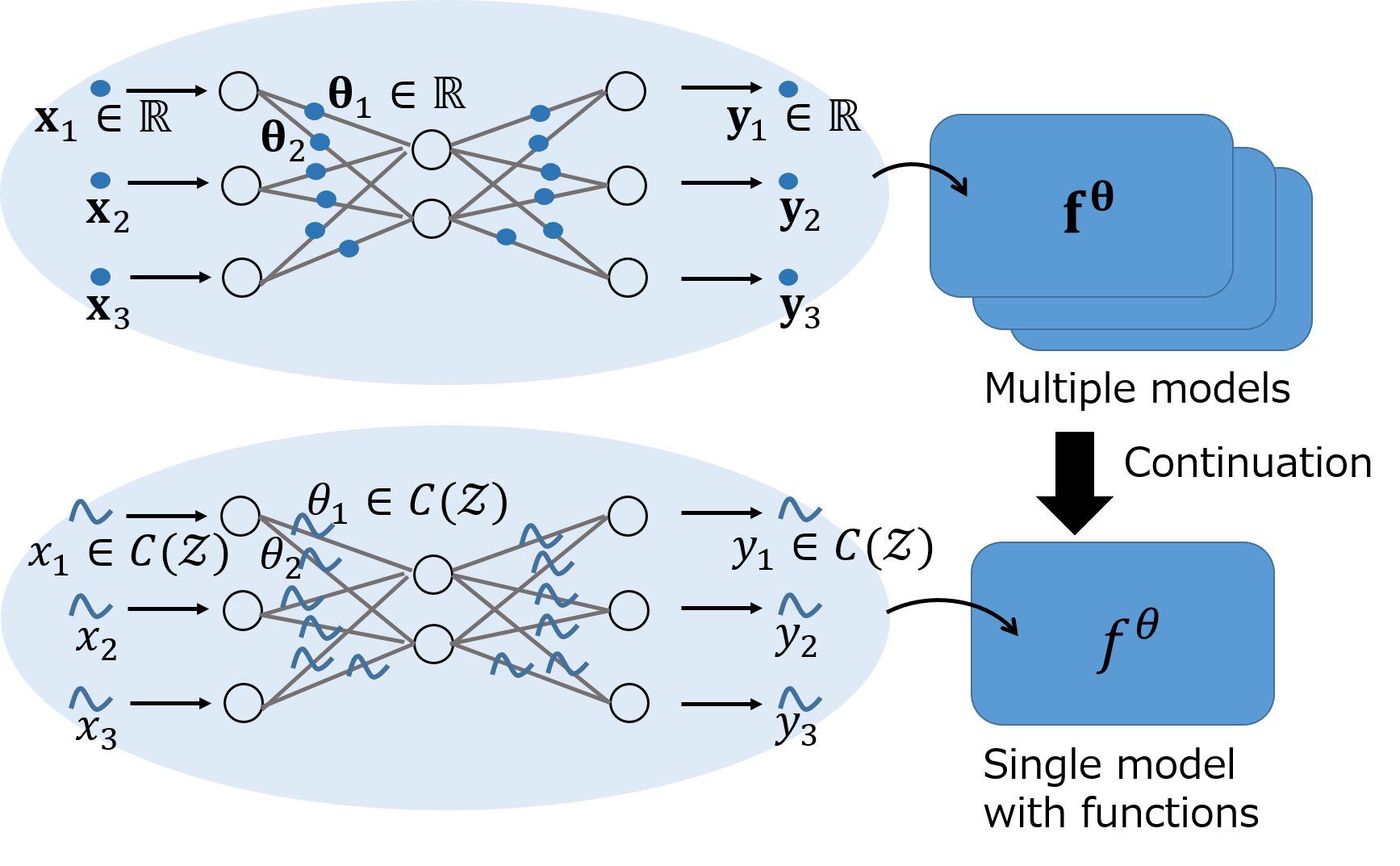}
\caption{Overview of our framework.}\label{fig:overview}
\end{figure}

\section{Motivation of applying $C^*$-algebra to neural networks}
Because $C^*$-algebra and Hilbert $C^*$-module are natural generalizations of the space of complex numbers and Hilbert space, we can naturally generalize neural networks by using $C^*$-algebra.
Let $\alg$ be the $C^*$-algebra of the space of continuous functions on a compact space.
Then, by generalizing the real-valued parameters in the neural networks to $C^*$-algebra-valued ones, we can combine multiple real-valued models continuously.
For example, we can apply our framework to ensemble learning.
In standard ensemble learning, we learn multiple models separately and aggregate the results of the models after finishing the learning process.
This approach may not be efficient since the multiple models do not interact during the learning process, although each model is learned for the same or related task.
On the other hand, by using the $\alg$-valued parameters, we introduce a continuous dependence between different models and can learn the multiple models with interactions.
As a result, we expect that our method outperforms the standard ensemble learning.
For further details of applications of our framework to practical situations, see Section~\ref{sec:application}.
\color{black}

\section{Background}\label{sec:background}
In this section, we review mathematical notions required for the remaining part of this paper.
In Subsection~\ref{subsec:c_algebra}, we review $C^*$-algebra, and in Subsection~\ref{subsec:c_module}, we review Hilbert $C^*$-module.
\red{All the definitions and examples in this section are standard terminologies in $C^*$-algebra, and they are adopted from~\citet{hashimoto21}.}
\red{The standard definitions related to the definitions in this section is provided in Section~\ref{ap:definition}.}
For further details of $C^*$-algebra and Hilbert $C^*$-module, see~\cite{lance95,murphy90,hashimoto21}.

\subsection{$C^*$-algebra}\label{subsec:c_algebra}
$C^*$-algebra is a generalization of the space of complex numbers.
\begin{definition}[$C^*$-algebra]~\label{def:c*_algebra}
A set $\mcl{A}$ is called a {\em $C^*$-algebra} if it satisfies the following conditions:

\begin{enumerate}
 \item $\mcl{A}$ is an algebra (See Definition~\ref{def:algebra}) over $\mathbb{C}$ and {equipped with} a bijection $(\cdot)^*:\mcl{A}\to\mcl{A}$ that satisfies the following conditions for $\alpha,\beta\in\mathbb{C}$ and $c,d\in\mcl{A}$:

 $\bullet$ $(\alpha c+\beta d)^*=\overline{\alpha}c^*+\overline{\beta}d^*$,\\
 $\bullet$ $(cd)^*=d^*c^*$,\qquad
 $\bullet$ $(c^*)^*=c$.

 \leftskip=0pt
 \item $\mcl{A}$ is a normed space with $\Vert\cdot\Vert$, and for $c,d\in\mcl{A}$, $\Vert cd\Vert\le\Vert c\Vert\,\Vert d\Vert$ holds.
 In addition, $\mcl{A}$ is complete with respect to $\Vert\cdot\Vert$.

 \item For $c\in\mcl{A}$, $\Vert c^*c\Vert=\Vert c\Vert^2$ holds.
\end{enumerate}

We introduce important notions related to $C^*$-algebra.
\end{definition} 
\begin{definition}[Multiplicative identity]\label{def:multiplicative_identity}
The {\em multiplicative identity} of a $C^*$-algebra $\alg$ is the element $a\in\alg$ that satisfies $ac=ca=c$ for any $c\in\alg$.
We denote by $1_{\alg}$ the multiplicative identity of $\alg$.
\end{definition}
\begin{definition}[Positive]~\label{def:positive}
An element $c$ of $\mcl{A}$ is called {\em positive} if there exists $d\in\mcl{A}$ such that $c=d^*d$ holds.
We denote by $\alg_+$ the subset of $\alg$ composed of all positive elements in $\alg$.
\end{definition}

An important example of $C^*$-algebras is the space of continuous functions on a compact space, \red{on which we focus in this paper.}
\begin{example}\label{ex:continuous}
Let $\mcl{Z}$ be a compact space and let $C(\mcl{Z})$ be the Banach space of continuous functions equipped with the sup norm.
Let $\cdot:C(\mcl{Z})\times C(\mcl{Z})\to C(\mcl{Z})$ be defined as $(c\cdot d)(z) = c(z)d(z)$ for $c,d\in C(\mcl{Z})$ and $z\in\mcl{Z}$.
In addition, let $(\cdot)^*:C(\mcl{Z})\to C(\mcl{Z})$ be defined as $c^*(z)=\overline{c(z)}$ for $c\in C(\mcl{Z})$.
Then, $C(\mcl{Z})$ is a $C^*$-algebra.
The multiplicative identity is the constant function whose value is $1$ at any $z\in\mcl{Z}$.
For $c\in C(\mcl{Z})$, $c$ is positive if and only if $c(z)\ge 0$ for any $z\in\mcl{Z}$.
\end{example}

\subsection{Hilbert $C^*$-module}\label{subsec:c_module}
Hilbert $C^*$-module is a generalization of Hilbert space.
We first introduce $C^*$-module, which is a generalization of vector space.
Then, we introduce $\alg$-valued inner product and Hilbert $C^*$-module.
\begin{definition}[$C^*$-module]\label{def:c*module}
Let $\modu$ be an abelian group with operation $+$ and let $\alg$ be a $C^*$-algebra.
If $\modu$ {is equipped with} an $\alg$-multiplication (See Definition~\ref{def:multiplication}), $\modu$ is called a {\em $C^*$-module} over $\alg$.
\end{definition}
\begin{definition}[$\alg$-valued inner product]\label{def:innerproduct}
Let $\alg$ be a $C^*$-algebra and let $\modu$ be a $C^*$-module over $\alg$.
A {$\mathbb{C}$-linear map with respect to the second variable} $\blacket{\cdot,\cdot}:\modu\times\modu\to\alg$ is called an $\alg$-valued {\em inner product} if it satisfies the following conditions for $u,v,p\in\modu$ 
and $c,d\in\alg$:

 $\bullet$ $\blacket{u,vc+pd}=\blacket{u,v}c+\blacket{u,p}d$,\quad
 $\bullet$ $\blacket{v,u}=\blacket{u,v}^*$,\\
 $\bullet$ $\blacket{u,u}$ is positive,\quad
 $\bullet$ If $\blacket{u,u}=0$ then $u=0$.
\end{definition}
\begin{definition}[Norm]\label{def:absolute_norm}
Let $\alg$ be a $C^*$-algebra and let $\modu$ be a $C^*$-module over $\alg$ equipped with an $\alg$-valued inner product $\blacket{\cdot,\cdot}$.
The (real-valued) norm $\Vert \cdot\Vert$ on $\modu$ is defined by $\Vert u\Vert =\big\Vert\blacket{u,u}\big\Vert_\alg^{1/2}$, where $\Vert\cdot\Vert_{\alg}$ is the norm in $\alg$. 
\end{definition}
\begin{definition}[Hilbert $C^*$-module]\label{def:hil_c*module}
Let $\alg$ be a $C^*$-module.
Let $\modu$ be a $C^*$-module over $\alg$ equipped with an $\alg$-valued inner product defined in Definition~\ref{def:innerproduct}.
If $\modu$ is complete with respect to the norm $\Vert \cdot\Vert$ defined in Definition~\ref{def:absolute_norm}, it is called a {\em Hilbert $C^*$-module} over $\alg$ or {\em Hilbert $\alg$-module}.
\end{definition}

\section{Neural network on $C^*$-algebra}\label{sec:cstar_net}
In this section, we propose a new framework with $C^*$-algebra that generalizes the existing framework of neural networks.
In Subsection~\ref{subsec:nn_hilbertsp}, we formulate the existing framework of neural networks.
Then, in Subsection~\ref{subsec:nn_hilbertmodu}, we generalize the existing framework to that on $C^*$-algebra.
\subsection{Existing framework of neural networks}\label{subsec:nn_hilbertsp}
We begin by formulating the existing framework of neural networks.
We consider a network with $H\in\mathbb{N}$ hidden layers.
Let $N_0,\ldots,{N_{H+1}}$ be natural numbers each of which represents the dimension of a layer.
Inputs are vectors in $\mathbb{R}^{N_0}$ and outputs are vectors in $\mathbb{R}^{N_{H+1}}$.
In addition, for $i=1,\ldots,H+1$, let $\mb{W}_i:\mathbb{R}^{N_{i-1}}\to\mathbb{R}^{N_{i}}$ be an $N_{i-1}\times N_i$ matrix and $\bs{\sigma}_i:\mathbb{R}^{N_{i}}\to\mathbb{R}^{N_{i}}$ be an (often nonlinear) activation function.
The neural network model $\mb{f}:\mathbb{R}^{N_0}\to\mathbb{R}^{N_{H+1}}$ is defined as
\begin{equation}
 \mb{f}=\boldsymbol{\sigma}_{H+1}\circ \mb{W}_{H+1}\circ\bs{\sigma}_{H}\circ \mb{W}_H\circ\cdots\circ\bs{\sigma}_1 \circ \mb{W}_1.\label{eq:nn}
\end{equation}
We fix $\bs{\sigma}_1,\ldots,\bs{\sigma}_{H+1}$ and find best possible matrices $\mb{W}_1,\ldots,\mb{W}_{H+1}$ by minimizing a loss function $\mb{L}$.
We regard the set of matrices $\mb{W}_1,\ldots,\mb{W}_{H+1}$ as an $N$-dimensional vector of parameters, which is denoted as $\bs{\theta}$, where $N=\sum_{i=1}^{H+1}N_{i-1}N_i$.
Then, we set a loss function $\mb{L}:\mathbb{R}^N\to\mathbb{R}_{+}$, which depends on the parameters (and usually on inputs and outputs).
Here, $\mathbb{R}_+$ is the set of all non-negative real numbers.
We learn the parameter $\bs\theta$ by minimizing the loss function.
The minimization of the loss function is implemented by a gradient descent method such as stochastic gradient descent (SGD) and Adam~\cite{kingma15}.
We compute the gradient $\nabla_{\bs{\theta}}\mb{L}$ of the loss function $\mb{L}$ and generate a sequence $\bs{\theta}_0,\bs{\theta}_1,\ldots$ using the gradient for finding a best possible $\bs{\theta}$. 

\subsection{Generalization to $C^*$-algebra}\label{subsec:nn_hilbertmodu}
\subsubsection{Formulation}
To improve the representational power of the model, we generalize the parameter $\theta$ on $\mathbb{R}^N$, which is a Hilbert space, to a Hilbert $C^*$-module.
Let $\alg$ be a commutative unital $C^*$-algebra.
By the Gelfand--Naimark theorem, there exists a compact Hausdorff space $\mcl{Z}$ such that $\alg$ is isometrically $*$-isomorphic to the $C^*$-algebra $C(\mcl{Z})$, the space of continuous functions on $\mcl{Z}$ (see Example~\ref{ex:continuous}).
Therefore, we focus on the case of $\alg=C(\mcl{Z})$ for a compact Hausdorff space $\mcl{Z}$ in the remaining parts of this paper.
As before, let $N_0,\ldots,{N_{H+1}}$ be natural numbers each of which represents the dimension of a layer.
However, in this case, we consider $\alg^{N_0}$-valued inputs and $\alg^{N_{H+1}}$-valued outputs.
Since $\alg$ is a function space, the inputs and outputs should be functions, which enables us to analyze functional data.
On the other hand, for scalar-valued data $\mb{x}\in\mathbb{R}$, we can transform $\mb{x}$ into an appropriate function such as the constant function ${x}\equiv \mb{x}$ (see Subsections~\ref{subsec:density_estimation} and \ref{subsec:few-shot} for practical applications).
In addition, for $i=1,\ldots,H+1$, let $W_i:\alg^{N_{i-1}}\to\alg^{N_{i}}$ be an $N_{i-1}\times N_i$ $\alg$-valued matrix and $\sigma_i:\alg^{N_{i}}\to\alg^{N_{i}}$ be an (often nonlinear) activation function.
The neural network model $f:\alg^{N_0}\to\alg^{N_0}$ is defined as
\begin{equation}
 f=\sigma_{H+1}\circ W_{H+1}\circ\sigma_{H}\circ W_H\circ\cdots\circ\sigma_1 \circ W_1\label{eq:nn_cstar}
\end{equation}
in the same manner as Eq.~\eqref{eq:nn}.
We regard the set of $\alg$-valued matrices $W_1,\ldots,W_{H+1}$ as an $N$-dimensional $\alg$-valued vector of parameters, which is denoted as $\theta$, where $N=\sum_{i=1}^{H+1}N_{i-1}N_i$.
Then, we set an $\alg$-valued loss function $L:\alg^N\to\alg_{+}$, which depends on the $\alg$-valued parameters.
Here, $\alg_+$ is the set of all positive elements in $\alg$ (see Definition~\ref{def:positive}).

\subsubsection{Learning $C^*$-algebra-valued parameters}\label{subsec:gd_cstar}
To implement a gradient descent method on $\alg^N$ and minimize the $\alg$-valued loss function, {we propose a practical approach to applying the gradient descent method proposed by \citet{hashimoto21} to our case. }
\red{In more detail, we add the regularization term and simultaneously learn multiple models with interactions.}
We first define a gradient of $\alg$-valued functions on $\alg^N$.
\begin{definition}[$\alg$-valued gradient]\label{def:gradient}
Let $L:\alg^N\to\alg$ be an $\alg$-valued function defined on $\alg^N$ and let $\theta\in\alg^N$.
Assume there exists $\xi\in\alg^N$ such that for any $\delta\in\alg^N$ and any $z\in\mcl{Z}$,
\begin{equation*}
 \lim_{\delta\to 0,\ \delta\in\alg^N\setminus\{0\}}\frac{L(\theta+\delta)(z)-L(\theta)(z)-\blacket{\xi,\delta}(z)}{\Vert\delta\Vert}=0.
\end{equation*}
In this case, we define $\xi$ as the {\em $\alg$-valued gradient} of $L$ at $\theta$ and denote it by $\nabla_{\alg,\theta} L$.
\end{definition}

\begin{example}
 Assume there exists a function $\tilde{L}:\mathbb{R}^N\times \mcl{Z}\to\mathbb{R}$ such that $L(\theta)(z)=\tilde{L}(\theta(z),z)$, that is, we can decompose $L$ into $\mathbb{R}$-valued functions indexed by $z$ on $\mathbb{R}^N$.
The function $\tilde{L}(\cdot,z)$ corresponds to an $\mathbb{R}$-valued (standard) loss function at $z\in\mcl{Z}$.
Assume $\tilde{L}(\cdot,z)$ has the (standard) gradient $\nabla_{\theta(z)}\tilde{L}(\cdot,z)\in\mathbb{R}^N$ for each $z\in\mcl{Z}$.
If the map $z\mapsto\nabla_{\theta(z)}\tilde{L}(\cdot,z)$ is contained in $\alg^N$, then it is the $\alg$-valued gradient.
\end{example}

We generate a sequence $\theta_0,\theta_1\ldots$ using the $\alg$-valued gradient for finding a best possible $\theta$ to minimize the $\alg$-valued loss function $L$.
The basic gradient descent scheme is 
\begin{equation}
\begin{aligned}
  &\theta_0\in\alg^N,\\
&\theta_{t+1}=\theta_{t}-\nabla_{\alg,\theta_t}L\cdot \eta_t\ (t=0,1,\ldots),
\end{aligned}\label{eq:gd_scheme}
\end{equation} 
where $\eta_t\in\alg_+$ is the learning rate.

\begin{example}\label{ex:gradient}
 Assume there exists a function $\tilde{L}:\mathbb{R}^N\times \mcl{Z}\to\mathbb{R}$ such that $L(\theta)(z)=\tilde{L}(\theta(z),z)$.
Then, for each $z\in\mcl{Z}$, the scheme~\eqref{eq:gd_scheme} is reduced to
\begin{equation*}
\begin{aligned}
  &\theta_0(z)\in\mathbb{R}^N,\\
&\theta_{t+1}(z)=\theta_{t}(z)-\eta_t(z)\nabla_{\theta_t(z)}\tilde{L}(\cdot,z)\ (t=0,1,\ldots),
\end{aligned}
\end{equation*}
which is the standard gradient descent scheme on $\mathbb{R}^N$ with the learning rate $\eta_t(z)$.
Thus, computing the scheme~\eqref{eq:gd_scheme} is equivalent to computing the standard gradient descent scheme on $\mathbb{R}^N$ simultaneously for all $z\in\mcl{Z}$ in this case.
If the standard gradient descent at each $z\in\mcl{Z}$ generates a sequence $\theta_0(z),\theta_1(z),\ldots$ in $\mathbb{R}^N$ converging to some $\theta^*(z)$, then we obtain a sequence $\theta_0,\theta_1,\ldots$ in $\alg^N$ converging pointwise to the function $\theta^*$.
\end{example}

The above example implies that if $L(\theta)$ is defined as $L(\theta)(z)=\tilde{L}(\theta(z),z)$ for a function $\tilde{L}:\mathbb{R}^N\times \mcl{Z}\to\mathbb{R}$, then the scheme~\eqref{eq:gd_scheme} is computed without any interactions among variables $z\in\mcl{Z}$.
To learn the model with interactions among $z$, we assume $\mcl{Z}$ is a compact finite measure space, and add an $L_1$ regularization term to the loss function $L$ as
\begin{equation*}
 L{\opn{reg}}(\theta)=L(\theta) + \int_{z\in\mcl{Z}}L(\theta)(z)dz 1_{\alg}\cdot \lambda,
\end{equation*}
where $1_{\alg}$ is the multiplicative identity of $\alg$ ($1_{\alg}\equiv 1$, see Definition~\ref{def:multiplicative_identity}) and $\lambda\in\alg_+$ is a hyperparameter.
Since the regularization term is a constant function with respect to $z$, it affects uniformly on $L_{\opn{reg}}(\theta)$ at any $z\in\mcl{Z}$.
It has an effect of aggregating the gradient at each $z\in\mcl{Z}$ and adding the aggregated gradient to the gradient at each $z\in\mcl{Z}$. 
Indeed, since $\nabla_{\alg,\theta}L\in\alg^N$, each element of $\nabla_{\alg,\theta}L$ is integrable.
Thus, we have
\begin{equation*}
 \nabla_{\alg,\theta}L_{\opn{reg}}(\theta)= \nabla_{\alg,\theta}L + \int_{z\in\mcl{Z}}(\nabla_{\alg,\theta}L)(z)dz \odot\mathbf{1}_{\alg}\lambda,
\end{equation*}
where $\mathbf{1}_{\alg}=[1_{\alg},\ldots,1_{\alg}]^T\in\alg^N$ and $\odot$ represents the element-wise product of $\int_{z\in\mcl{Z}}(\nabla_{\alg,\theta})L(z)dz\in\mathbb{C}^N$ and $\mathbf{1}_{\alg}\in\alg^N$.

In practical computations, we cannot handle functions in the infinite-dimensional space $\alg$.
Therefore, we set a finite-dimensional subspace $\mcl{V}$ of $\alg^N$ and a map $P:\alg^N\to\mcl{V}$.
Moreover, we set the learning rate $\eta_t$ and the hyperparameter $\lambda$ as constant functions $\eta_t\equiv\tilde{\eta}_t$ and $\lambda\equiv\tilde{\lambda}$ for some $\tilde{\eta}_t,\tilde{\lambda}\in\mathbb{R}+$.
Then, the practical scheme is 
\begin{equation}
\begin{aligned}
  &\theta_0\in\mcl{V},\\
&\theta_{t+1}=\theta_{t}-\tilde{\eta}_tP(\nabla_{\alg,\theta_t}L)\\
&\quad-\tilde{\lambda}\int_{z\in\mcl{Z}}P(\nabla_{\alg,\theta}L)(z)dz \odot\mathbf{1}_{\alg}\ (t=0,1,\ldots).
\end{aligned}\label{eq:gd_scheme_mod}
\end{equation}
Typically, we set $P(\theta)$ as a regression of $\theta(z_1),\ldots,\theta(z_n)$ in $\mcl{V}$, where $z_1,\ldots,z_n$ for some $n\in\mathbb{N}$ are properly chosen fixed points.
As a result, the sequence $\theta_0,\theta_1,\ldots$ is contained in $\mcl{V}$, which enables us to compute the scheme~\eqref{eq:gd_scheme_mod} practically.
In addition, by applying $P$, the scheme becomes interactive with respect to $z\in\mcl{Z}$.
That is, $P(\theta)(z_0)$ is determined by using the values $\theta(z)$ for $z\neq z_0$.
Regarding the choice of $\mcl{V}$, we choose it on the basis of the following fact.
\begin{fact}\label{fact:uniform}
Let $\mcl{Z}$ be a compact metric space and let $d_{\mcl{Z}}$ be the metric on $\mcl{Z}$.
Let $\zeta_0,\zeta_1,\ldots\in\alg$ be a sequence such that there exists a function $\zeta^*$ on $\mcl{Z}$, $\lim_{t\to\infty}\zeta_t(z)=\zeta^*(z)$ for any $z\in\mcl{Z}$.
If $\zeta_0,\zeta_1,\ldots$ is uniformly Lipschitz continuous,
that is, there exists $C>0$ such that for any $t\in\mathbb{N}$, 
\begin{equation*}
\vert \zeta_t(z_1)-\zeta_t(z_2)\vert\le Cd_{\mcl{Z}}(z_1,z_2),
\end{equation*}
then $\zeta_0,\zeta_1,\ldots$ converges uniformly to $\zeta^*$ and $\zeta^*\in\alg$.
\end{fact}
On the basis of Fact~\ref{fact:uniform}, we set $\mcl{V}=\opn{Span}\{v_1,\ldots,v_l\}^N$ for $l\in\mathbb{N}$ and Lipschitz continuous functions $v_1,\ldots,v_l\in\alg$.
Then, each element $\theta_{t,j}\in\alg$ of $\theta_t$ is represented as $\sum_{i=1}^nc_{t,i}v_i$ for some $c_{t,i}\in\mathbb{C}$.
If there exists $C_i>0$ such that for any $t\in\mathbb{N}$ $\vert c_{t,i}\vert\le C_i$ for $i=1,\ldots,n$, then the sequence $\theta_{0,j},\theta_{1,j},\ldots$ is uniformly Lipschitz.
Thus, if it converges pointwise, the convergence is uniform.
Uniform convergence is a stronger notion than pointwise convergence and preserves properties of the sequence of functions to its limit.
For example, the uniform limit of a sequence of uniformly continuous functions is also uniformly continuous, but the pointwise limit of a sequence of continuous functions is not always continuous.
Thus, the uniform convergence is more suitable when we need the parameter $\theta$ to preserve properties as a function throughout the gradient descent.

\begin{remark}
The scheme~\ref{eq:gd_scheme_mod} is a basic scheme of the $\alg$-valued gradient descent.
We can also consider variants of the scheme such as SGD and Adam in the same manner as the standard gradient descent on $\mathbb{R}^N$.
\end{remark}

\paragraph{Computation complexity}
The computation complexity of learning the $C^*$-algebra net is $O(pl)$, where $p$ is the number of parameters of the network and $l$ is the number of basis functions in $\mcl{V}$. 
It is $l$ times as large as that of the corresponding $\mathbb{R}$-valued model.

We illustrate the difference between $\mathbb{R}$-valued and $\alg$-valued models for a specific case of the linear regression problem on $\mathbb{R}$ in Section~\ref{ap:diff}.
\color{black}

\section{Applications}\label{sec:application}
In this section, we show examples of practical applications of our framework.
In Subsection~\ref{subsec:density_estimation}, we apply our framework to density estimation with normalizing flow.
In Subsection~\ref{subsec:few-shot}, we apply our framework to few-shot learning.
By using a model with function-valued parameters, we can learn features of data efficiently even with a limited number of samples.
Our framework is general and its application is not restricted to the above two cases.
In Subsection~\ref{subsec:other_appl}, we show other examples of applications.
\subsection{Density estimation}\label{subsec:density_estimation}
\subsubsection{Density estimation with normalizing flow}\label{subsec:nf_standard}
We can improve density estimation by replacing parameters in a model by $\alg$-valued ones.
In this paper, we focus on using normalizing flow~\citep{dinh14,dinh17,teshima20}, but the application is not limited to normalizing flow.
We first review density estimation with normalization flow briefly.
Let $\Omega$ be a probability space.
We construct $\mb{f}=\mb{f}^{\bs{\theta}}$ in Eq.~\eqref{eq:nn} as a measurable, invertible, and differentiable map that transforms a normal distribution into the distribution of data.
That is, if the distribution of a random variable $X$ defined on $\Omega$ and taking its value in $\mathbb{R}^{N_0}$ is the normal distribution, then the distribution of $\mb{f}^{\bs{\theta}}(X)$ is the distribution of data.
The density $p^{\bs{\theta}}_{\opn{data}}$ of $\mb{f}^{\bs{\theta}}(X)$ is calculated as 
\begin{equation}
 p^{\bs{\theta}}_{\opn{data}}(\mb{x})=p_{\opn{n}}\big((\mb{f}^{\bs{\theta}})^{-1}(\mb{x})\big)\big\vert\;\opn{det}\big(\nabla_{\mb{x}}(\mb{f}^{\bs{\theta}})^{-1}\big)\big\vert\label{eq:change_val}
\end{equation}
for $x\in\mathbb{R}^{N_0}$, where $p_{\opn{n}}$ is the density of the normal distribution.
By using the formula~\eqref{eq:change_val}, we compute the likelihood of given samples $\mb{x}_1,\ldots,\mb{x}_n\in\mathbb{R}^{N_0}$ and set the loss function $\mb{L}:\mathbb{R}^N\to\mathbb{R}_+$ as a negative log likelihood:
\begin{equation*}
 \mb{L}(\bs{\theta})=-\sum_{i=1}^n\log p^{\bs{\theta}}_{\opn{data}}(\mb{x}_i).
\end{equation*}

\subsubsection{Generalization to $C^*$-algebra}\label{subsec:nf_cstar}
We generalize the above setting to that with $C^*$-algebra.
Let $\mcl{Z}$ be a compact probability space and let $X$ be a random variable defined on $\mcl{Z}\times\Omega$ and taking its value in $\mathbb{R}^{N_0}$.
Assume that for any $\omega\in\Omega$, $X(\cdot,\omega)$ is continuous on $\mcl{Z}$.
Let $\mcl{D}$ be the probability measure on $\mcl{Z}$.
The map $f=f^{\theta}$ in Eq.~\eqref{eq:nn_cstar} is constructed so that if for any $z\in\mcl{Z}$, the distribution of a random variable $X(z,\cdot)$ is a normal distribution on $\mathbb{R}^{N_0}$, then the distribution of the random variable $\omega\mapsto f^{\theta}(X(\cdot,\omega))(z)$ is the distribution of data.
The density $p^{\theta,z}_{\opn{data}}$ of the random variable $\omega\mapsto f^{\theta}(X(\cdot,\omega))(z)$ is calculated as 
\begin{equation}
 p^{\theta,z}_{\opn{data}}(\mb{x})=p_{\opn{n}}^z\big((f^{\theta})^{-1}({x})(z)\big)\;\big\vert\opn{det}\big(\nabla_{\alg,{x}}(f^{\theta})^{-1}(z)\big)\big\vert\label{eq:change_val_cstar}
\end{equation}
for $\mb{x}\in\mathbb{R}^{N_0}$, where ${x}\in\alg^{N_0}$ is the constant function ${x}\equiv \mb{x}$.
By using the formula~\eqref{eq:change_val_cstar}, we compute the likelihood of given samples $\mb{x}_1,\ldots,\mb{x}_n\in\mathbb{R}^{N_0}$ and set the loss function $L:\alg^N\to\alg_+$ as a negative log likelihood:
\begin{align}
&L(\theta)(z)=-\sum_{i=1}^n\log p^{\theta,z}_{\opn{data}}(\mb{x}_i).\label{eq:loss_nf} 
\end{align}
We apply the $\alg$-valued gradient descent method proposed in Subsection~\ref{subsec:gd_cstar} to minimize the loss function~\eqref{eq:loss_nf}.
Since $L(\theta)(z)$ depends only on $\theta(z)$, we have $L(\theta)(z)=\tilde{L}(\theta(z),z)$ for some function $\tilde{L}:\mathbb{R}^N\times \mcl{Z}\to\mathbb{R}$.
As we explained in Example~\ref{ex:gradient}, the $\alg$-valued gradient of $L$ is calculated by the standard gradient of $\tilde{L}(\cdot,z)$.
After learning the model, we obtain the distribution ${p}^{\theta,z}_{\opn{data}}$ at each $z\in\mcl{Z}$.
Ideally, ${p}^{\theta,z}_{\opn{data}}$ is independent of $z$ since the distribution of data is independent of $z$.
Practically, we get an estimation of the density $\tilde{p}^{\theta}_{\opn{data}}$ of data by integrating $p^{\theta,z}_{\opn{data}}$ as $\tilde{p}^{\theta}_{\opn{data}}(\mb{x})=\int_{z\in\mcl{Z}}p^{\theta,z}_{\opn{data}}(\mb{x})d\mcl{D}(z)$.

\begin{figure}[t]
\centering
 \subfigure[{Existing}]{\includegraphics[scale=0.27]{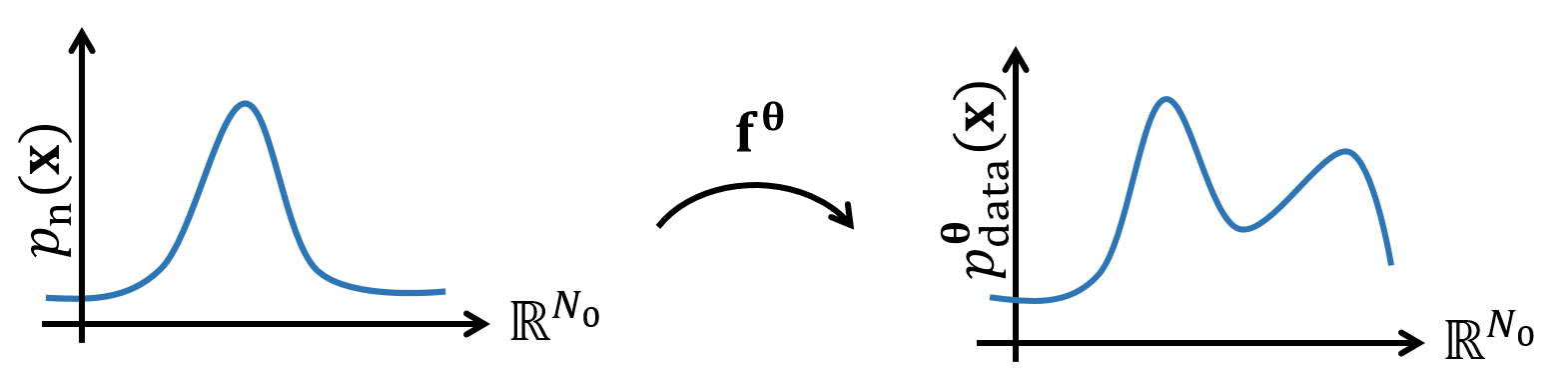}}
\subfigure[Ours]{\includegraphics[scale=0.27]{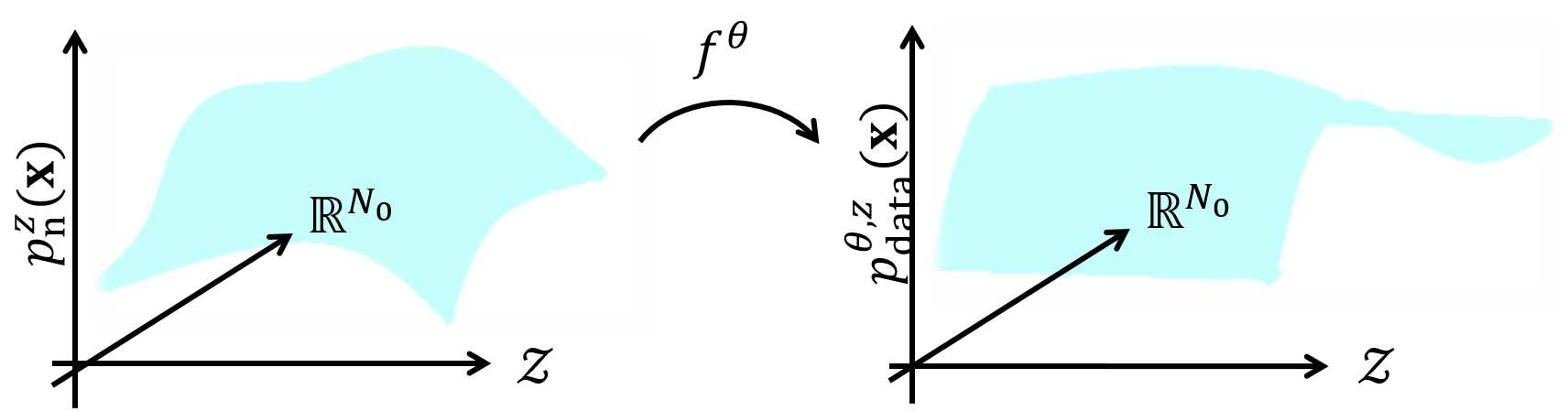}}
\caption{{Difference between the existing method and our method of density estimation with normalizing flow. (Although the set $\mcl{Z}$ in (b) is depicted as if it is a one-dimensional space, we can set $\mcl{Z}$ as an arbitrary compact space.)}}\label{fig:nf_overview}
\end{figure}

\begin{figure}[t]
\centering
 \subfigure[Swiss roll]{\includegraphics[scale=0.15]{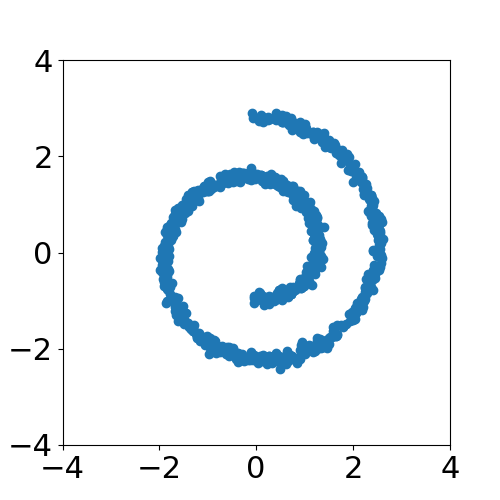}}\ 
\subfigure[Circles]{\includegraphics[scale=0.15]{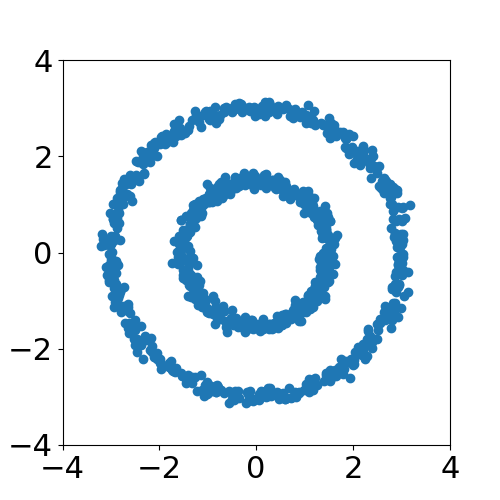}}\ 
 \subfigure[Swiss roll {\small (training)}]{\includegraphics[scale=0.15]{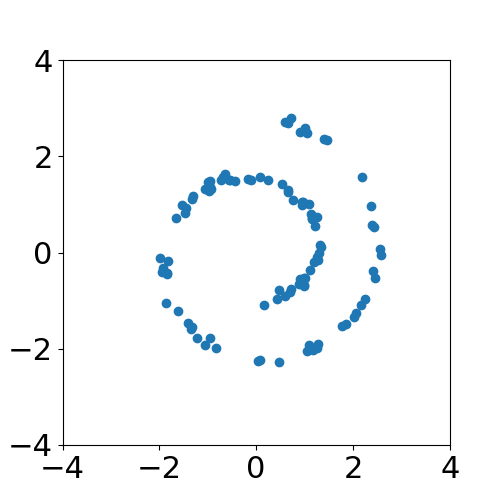}}\ 
\subfigure[Circles {\small (training)}]{\includegraphics[scale=0.15]{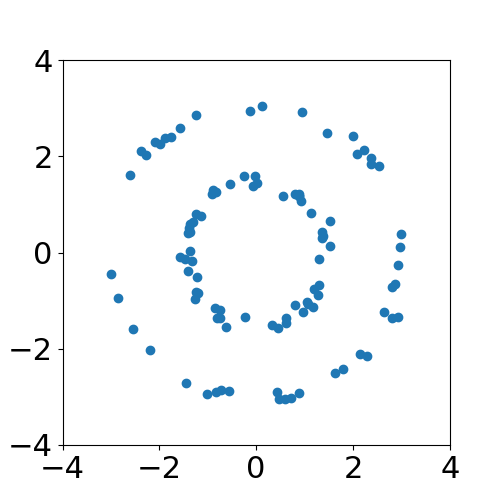}}
\caption{{Samples in the datasets ((c) and (d) show training samples)}}\label{fig:samples}
\end{figure}

\subsubsection{Numerical results}\label{subsec:nf_numexp}
\begin{figure}[t]
\centering
\newcolumntype{C}{>{\centering\arraybackslash}X}
\subfigure[Swiss roll]{
\begin{tabularx}{\linewidth}{cCCC}
\setlength{\tabcolsep}{0mm} 
&
\multirow{5}{*}{\includegraphics[scale=0.18]{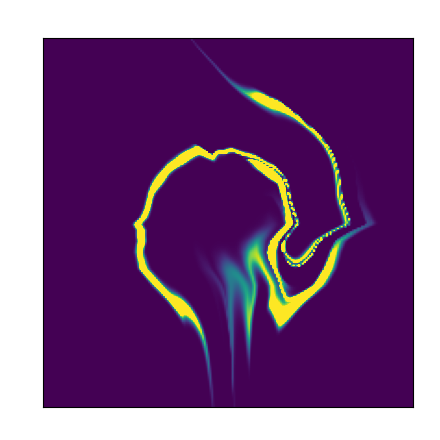}}&
\multirow{5}{*}{\includegraphics[scale=0.18]{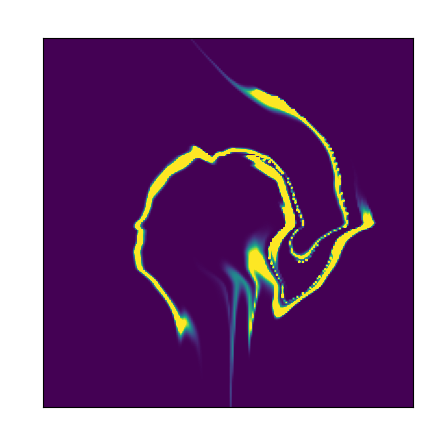}}&
\multirow{5}{*}{\includegraphics[scale=0.18]{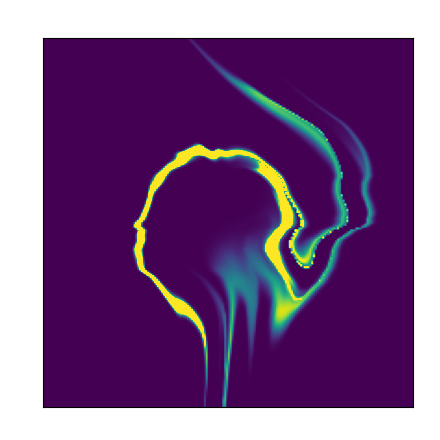}}\\
&&&\\
\footnotesize \underline{Standard}&&&\\
\footnotesize Density&&&\\
&&&\\
\footnotesize NLL&\footnotesize \ $3.68\pm 0.413$ & \footnotesize\ $4.07\pm 0.340$ & \footnotesize\ $3.91\pm 0.621$\\
&
\multirow{5}{*}{\includegraphics[scale=0.18]{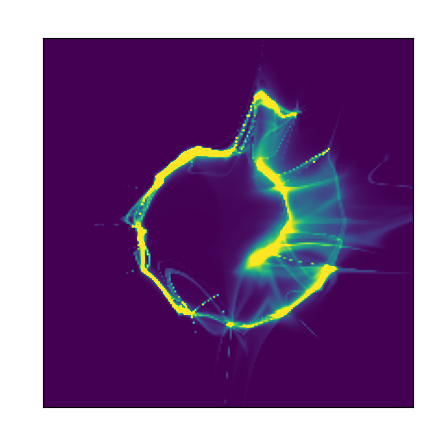}}&
\multirow{5}{*}{\includegraphics[scale=0.18]{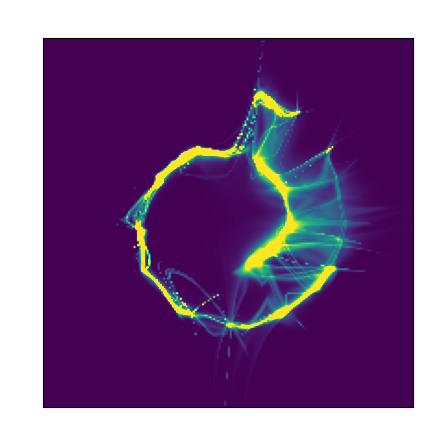}}&
\multirow{5}{*}{\includegraphics[scale=0.18]{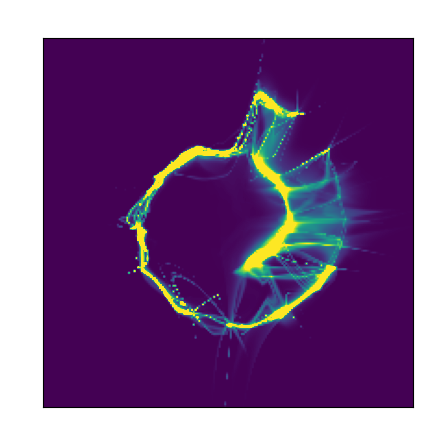}}\\
&&&\\
\footnotesize {\underline{Discrete}}&&&\\
\footnotesize Density&&&\\
&&&\\
\footnotesize NLL&\footnotesize \ $4.09\pm 0.341$ & \footnotesize\ $4.00\pm 0.198$ & \footnotesize\ $4.03\pm 0.268$\\
&
\multirow{5}{*}{\includegraphics[scale=0.18]{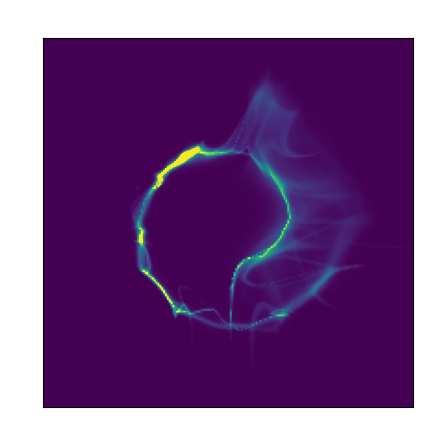}}&
\multirow{5}{*}{\includegraphics[scale=0.18]{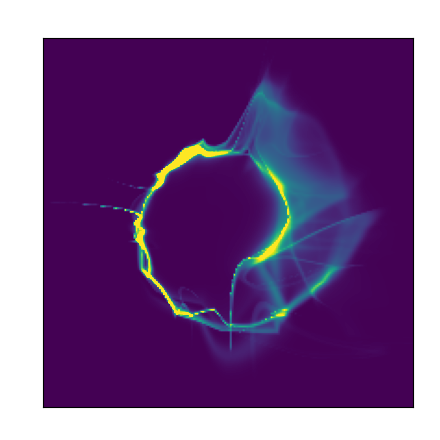}}&
\multirow{5}{*}{\includegraphics[scale=0.18]{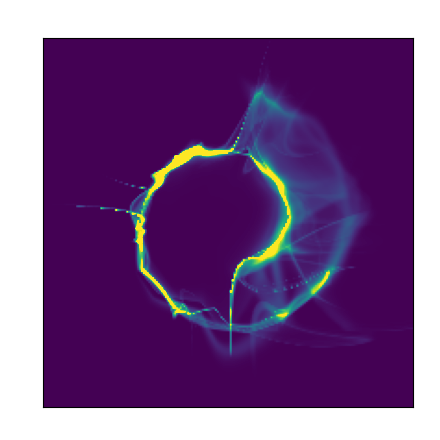}}\\
&&&\\
\footnotesize \underline{Ours}&&&\\
\footnotesize Density&&&\\
&&&\\
\footnotesize NLL&\footnotesize \ $3.75\pm 0.022$ & \footnotesize$\mb{2.97\pm 0.029}$ & \footnotesize$\mb{2.94\pm 0.057}$\\
&\footnotesize $2000$ epochs & \footnotesize {$2500$ epochs} & \footnotesize {$3000$ epochs}
\end{tabularx}}
\subfigure[Circles]{
\begin{tabularx}{\linewidth}{cCCC}
\setlength{\tabcolsep}{0mm} 
&
\multirow{5}{*}{\includegraphics[scale=0.18]{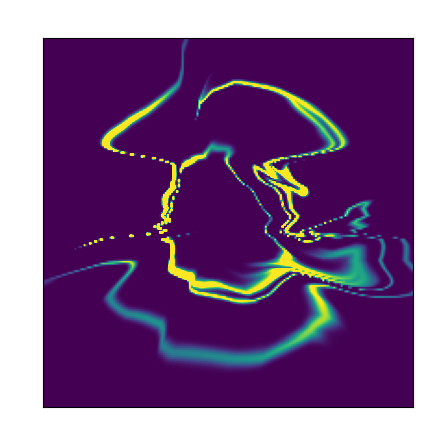}}&
\multirow{5}{*}{\includegraphics[scale=0.18]{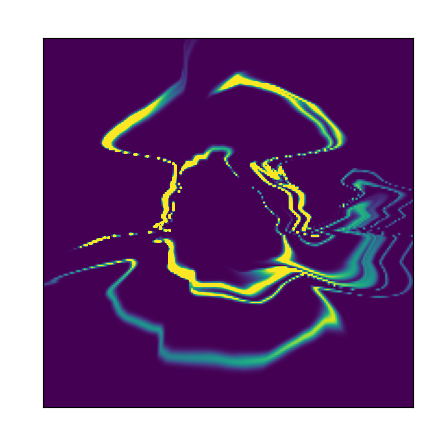}}&
\multirow{5}{*}{\includegraphics[scale=0.18]{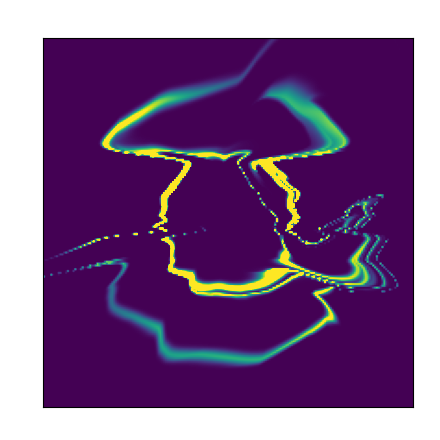}}\\
&&&\\
\footnotesize \underline{Standard}&&&\\
\footnotesize Density&&&\\
&&&\\
\footnotesize NLL&\footnotesize \ $4.81\pm 0.211$ & \footnotesize\ $4.73\pm 0.556$ & \footnotesize\ $4.62\pm 0.254$\\
&
\multirow{5}{*}{\includegraphics[scale=0.18]{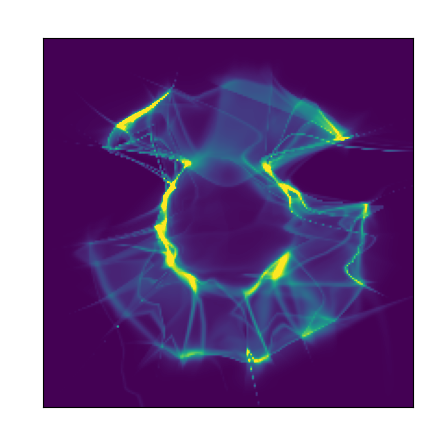}}&
\multirow{5}{*}{\includegraphics[scale=0.18]{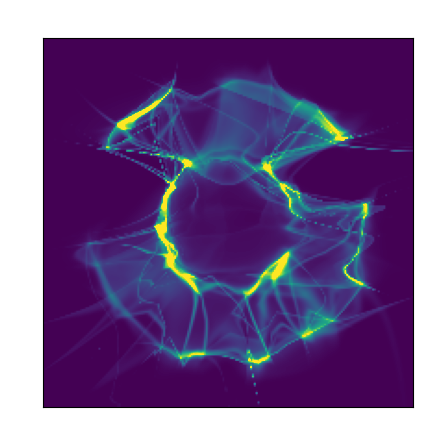}}&
\multirow{5}{*}{\includegraphics[scale=0.18]{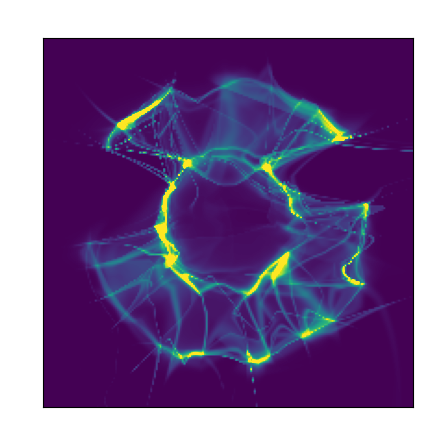}}\\
&&&\\
\footnotesize {\underline{Discrete}}&&&\\
\footnotesize Density&&&\\
&&&\\
\footnotesize NLL&\footnotesize \ $4.91\pm 0.129$ & \footnotesize\ $5.00\pm 0.135$ & \footnotesize\ $5.14\pm 0.200$\\
&
\multirow{5}{*}{\includegraphics[scale=0.18]{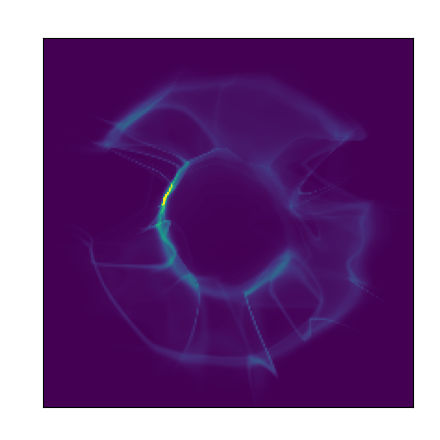}}&
\multirow{5}{*}{\includegraphics[scale=0.18]{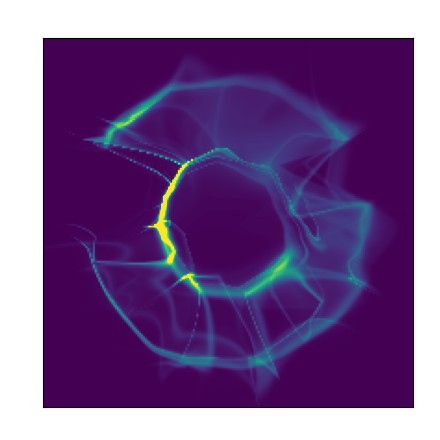}}&
\multirow{5}{*}{\includegraphics[scale=0.18]{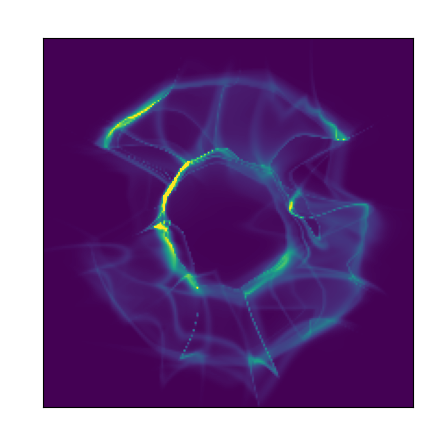}}\\
&&&\\
\footnotesize \underline{Ours}&&&\\
\footnotesize Density&&&\\
&&&\\
\footnotesize NLL&\footnotesize \ $4.97\pm 0.210$ & \footnotesize$\mb{4.17\pm 0.245}$ & \footnotesize$\mb{4.17\pm 0.157}$\\
&\footnotesize $2000$ epochs & \footnotesize {$2500$} epochs & \footnotesize {$3000$ epochs}
\end{tabularx}}
\caption{{Comparison among estimated densities and negative log-likelihoods (NLLs) with existing methods and our method. (Each value of negative log likelihood is an average ($\pm$ a standard deviation) over 5 independent runs.)}}\label{fig:density}
\end{figure}
We show the validity of applying our method to density estimation with normalizing flow numerically.
We generated two toy datasets: swiss roll and circles.
Each dataset contains $1000$ samples, illustrated in Fig.~\ref{fig:samples}.
We put $100$ samples for training samples from the datasets.
The estimation of the densities is challenging since the amount of training samples is small. 
We used masked autoregressive flow~\citep{papamakarios17} to construct the invertible differentiable map $f^{\theta}$, and for the gradient descent method, we used Adam.
We set the learning rate so that it decays polynomially starting from $0.001$ with decay rate $0.5$.
We set $\mcl{Z}=[-4,4]\times [-4,4]\subseteq \mathbb{R}^2$ and set the finite-dimensional subspace $\mcl{V}$ of $\alg^N$ as $\opn{Span}\{v_1,\ldots,v_l\}^N$, where $l=9$, $v_i(z)=e^{-10\Vert z_i-z\Vert^2}$, $z_1=[0,0]$, and $z_{4i+j+1}=[(2+i)\sin(2\pi (j-1+0.5i)/4),(2+i)\cos(2\pi (j-1+0.5i)/4)]$ for $i=0,1$ and $j=1,\ldots,4$.
The density $p_{\opn{n}}(z,\cdot)$ of the base normal distribution at $z\in\mcl{Z}$ is set as the normal distribution with mean $z$ and standard deviation $1$.
We set the map ${P}:\alg\to\mcl{V}$ as the kernel ridge regression, where the $i$th element of ${P}(\theta)$ for ${\theta}=[\theta_1,\ldots,\theta_N]^T\in\alg^N$ is computed as $[v_1,\ldots,v_n](G+\mu I)^{-1}[{\theta}_i(z_1),\ldots,{\theta}_i(z_l)]^T$.
Here, $G$ is the Gram matrix whose $(i,j)$-element is $e^{-10\Vert z_i-z_j\Vert^2}$ and $\mu\in\mathbb{R}_+$ is a hyperparameter of the regression {that controls the strength of a regularization term of the regression}.
In this experiment, we set $\mu=0.1$.
In addition, we set the hyperparameter $\tilde{\lambda}$ in Eq.~\eqref{eq:gd_scheme_mod} as $0.3$ and set the distribution $\mcl{D}$ on $\mcl{Z}$ as the uniform distribution on $\bigcup_{i=1}^9\{z\in\mcl{Z}\ \mid\ \Vert z-z_i\Vert\le 0.05\}$.
We compared our method proposed in Subsection~\ref{subsec:nf_cstar} with two straightforward methods, (1) the standard method explained in Subsection~\ref{subsec:nf_standard}, called ``standard'', and (2) learning the model whose base normal distribution is $p_{\opn{n}}(z_i,\cdot)$ for $i=1,\ldots 9$ separately and computing the mean value of the results, called {``discrete''}.
Note that the method (1) corresponds to the case of $l=1$ and $\tilde{\lambda}=\mu=0$ and the method (2) corresponds to the case of $l=9$ and $\tilde{\lambda}=\mu=0$ of the proposed method. 

Fig.~\ref{fig:density} shows the estimated densities {and negative log-likelihoods computed with the $900$ samples other than the training samples in each dataset}. 
{We see that our method with 2500 or 3000 epochs gives the best estimation of the densities, and its negative log-likelihoods are the smallest for each dataset.}
Since the amount of samples is small, the standard method fails to capture the detailed shape of the density {and its shape is blurred}.
By learning multiple models separately and computing the mean value, the shape becomes more clear, but not smooth.
Our method captures the shape clearly and smoothly because it allows us to learn multiple models simultaneously with interactions by using smooth functions.

\paragraph{Computation complexity and memory usage}
The method ``discrete'' learns $9$ real-valued models separately, and in the above experiment, we compared our proposed method with $l=9$ to the method ``discrete''.
The computation complexities of these two methods are the same.
Precisely, in our method, we compute the approximation of elements in $\mathcal{A}$ when updating the parameters. (The operator $P$ in Eq. (4))
However, the computation complexity of this part can be ignored since we used the kernel ridge regression and $l\ll p$.
The approximation of a function $a\in\mathcal{A}$ is computed by just the multiplication of an $l\times l$ fixed matrix to the vector $[a(z_1),\ldots,a(z_l)]$. 
Our method outperforms the method ``Discrete'' since our method properly combines multiple models continuously.
Regarding the memory usage, the amount of memory only increases linearly with respect to $l$.

\color{black}

\subsection{Few-shot learning}\label{subsec:few-shot}
\subsubsection{Few-shot learning for classification}\label{subsec:few_shot_form}
We can improve the accuracy of few-shot learning by replacing parameters in a model by $\alg$-valued ones.
In this paper, we focus on classification tasks, but the application of our framework is not limited to classification tasks.
Few-shot learning challenges a model to be learned with a limited number of samples~\citep{fei-fei06,lake11}.
We first formulate the setting for a supervised classification task briefly.
We construct $\mb{f}=\mb{f}^{\bs{\theta}}$ in Eq.~\eqref{eq:nn} as a map which maps a sample such as an image to its label, that is, the probability to belong to each class.
For given samples $\mb{x}_1,\ldots,\mb{x}_n\in\mathbb{R}^{N_0}$ and their labels $\mb{y}_1,\ldots,\mb{y}_n\in\mathbb{R}^{N_{H+1}}$, we set the loss function $\mb{L}:\mathbb{R}^N\to\mathbb{R}_+$ as the cross categorical entropy:
\begin{equation*}
 \mb{L}(\bs{\theta})=-\sum_{i=1}^n\blacket{\mb{y}_i,\log(\mb{f}^{\bs{\theta}}(\mb{x}_i))}.
\end{equation*}
For few-shot learning, one approach to learning the model effectively is meta-learning~\citep{ravi17,finn17,rusu19}.
Before learning a model for a new task $\mcl{T}_{\opn{new}}$, we learn a meta-model with multiple tasks $\mcl{T}_1,\ldots,\mcl{T}_m$ to extract common features among all tasks.
For example, \citet{rusu19} propose to learn a model to obtain a map $Z$ that maps a task $\mcl{T}_i$ to its low-dimensional representation $z_i$ and a map $\Theta$ that maps $z_i$ into the parameter of a model specific to the task $\mcl{T}_i$.
Then, for a given new task $\mcl{T}_{\opn{new}}$, we can get a corresponding low-dimensional representation $z_{\opn{new}}=Z(\mcl{T}_{\opn{new}})$ and a parameter $\Theta(z_{\opn{new}})$.
By using $\Theta(z_{\opn{new}})$ as the initial value of the parameter $\theta$, we can learn the model for the new task $\mcl{T}_{\opn{new}}$ efficiently even with a limited number of samples for the new task.

\begin{figure}[t]
\centering
 \subfigure[{Existing}]{\includegraphics[scale=0.27]{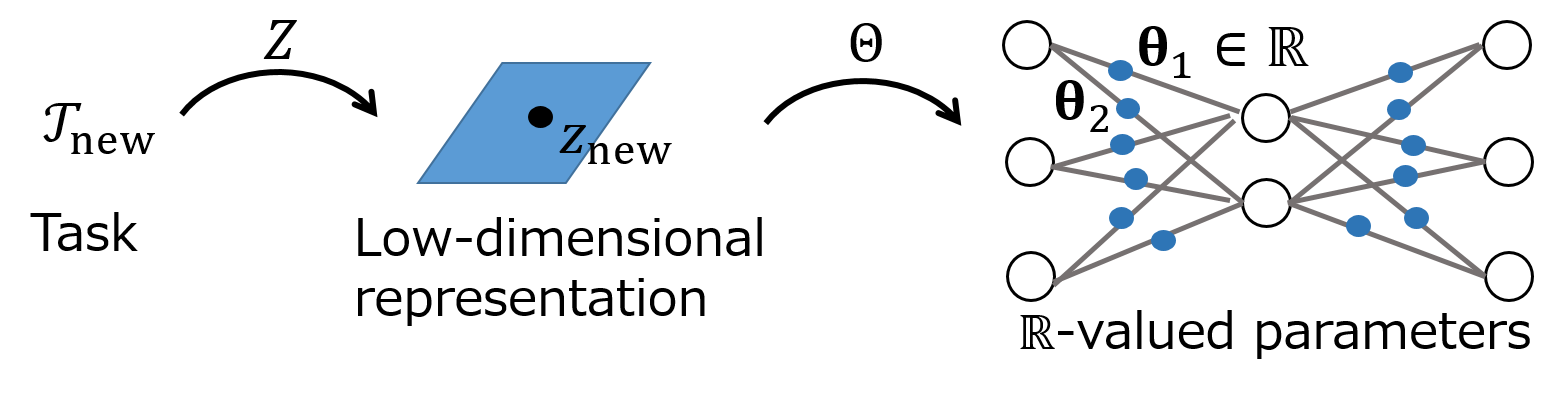}}
\subfigure[Ours]{\includegraphics[scale=0.27]{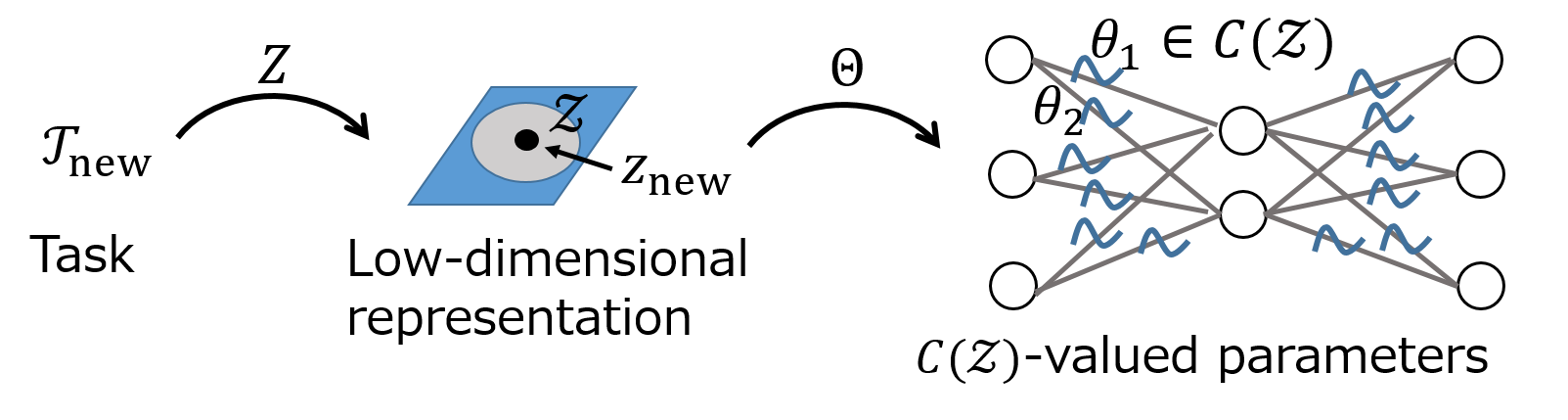}}
\caption{{Difference between the existing method and our method for the transformation of a task to parameters of the model.}}\label{fig:nf_overview}
\end{figure}

\subsubsection{Generalization to $C^*$-algebra}
We generalize the above setting to that with $C^*$-algebra.
We construct $f=f^{\theta}$ in Eq.~\eqref{eq:nn_cstar} as a map that maps the constant function ${x}\equiv \mb{x}$ for a sample $\mb{x}$ to the constant function ${y}\equiv \mb{y}$ for its label $\mb{y}$.
As described in Section~\ref{subsec:few_shot_form}, for a given new task $\mcl{T}_{\opn{new}}$, we can get a corresponding low-dimensional representation $z_{\opn{new}}=Z(\mcl{T}_{\opn{new}})$.
However, in this case, we use also vectors in the neighborhood of $z_{\opn{new},1}$.
We set $\mcl{Z}$ as a neighborhood of $z_{\opn{new}}$ and set the $\alg$-valued parameter $\theta$ as a function on $\mcl{Z}$.
For given samples $\mb{x}_1,\ldots,\mb{x}_n\in\mathbb{R}^{N_0}$ and its labels $\mb{y}_1,\ldots,\mb{y}_n\in\mathbb{R}^{N_{H+1}}$, we set the loss function $L:\mathbb{R}^N\to\mathbb{R}_+$ as the cross categorical entropy:
\begin{equation*}
 L(\theta)=-\sum_{i=1}^n\blacket{{y}_i,\log(f^{\theta}({x}_i))}.
\end{equation*}
By using the restriction of the function $\Theta$ on $\mcl{Z}$ as the initial value of the $\alg$-parameter $\theta$, we can learn the model for the new task $\mcl{T}_{\opn{new}}$ using the information about tasks distributed in the neighborhood of the new task in the space of the low-dimensional representations.
This makes learning the model more efficient even with a limited number of samples for the new task.

\begin{figure}[t]
\centering
 \subfigure[Task $1$]{\includegraphics[scale=0.27]{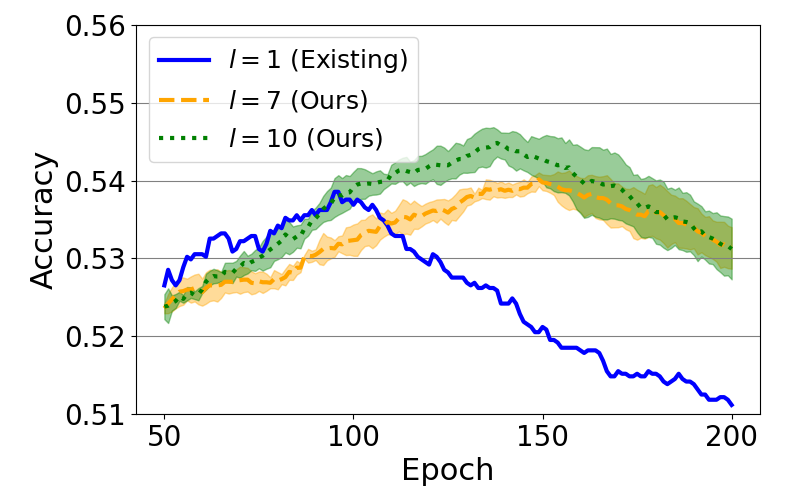}}\vspace{-.2cm}\\
\subfigure[Task $2$]{\includegraphics[scale=0.27]{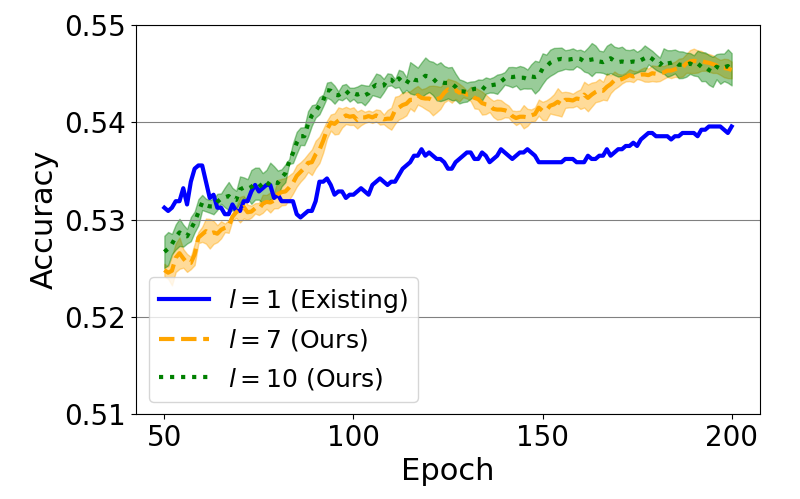}}\vspace{-.2cm}
\caption{{Test accuracies of the new tasks with different values of $l$. (Each value is an average ($\pm$ a standard deviation) over 5 independent runs with different values of $z_{j,i}$ ($j=1,\ldots,5$, $i=2,\ldots,l$). $\mu=0.05$ for Task 1 and $\mu=0.1$ for Task 2.)}}\label{fig:fs_accuracy}
\end{figure}

\subsubsection{Numerical results}
We show the validity of applying our method to few-shot learning numerically.
We used the miniImageNet dataset~\citep{vinyals16}, which is composed of $100$ classes, each of which has $600$ images, and considered the $5$-way $1$-shot task in the same manner as Subsection 4.2 in~\citep{rusu19}.
We randomly split the dataset into the train data with $1$ image and the test data with $599$ images.
In this case, each task is the classification with respect to randomly selected $5$ classes with $1$ sample in each class.
We set $f^{\theta}$ as the one-layer softmax classifier as used by~\citet{rusu19}, and for the gradient descent method, we used Adam with learning rate $0.001$.
In this case, the number of parameters $N$ is $N_1N_0$, where $N_0=640$ and $N_1=5$.
We first learn the model proposed by~\citet{rusu19} and obtain maps $Z$ and $\Theta$ explained in Subsection~\ref{subsec:few_shot_form}.
In their model, the dimension of the low-dimensional representation is $320$ and for $j=1,\ldots,N_1$, elements $N_0(j-1)+1\sim N_0j$ of $\Theta$ only depends on elements $64(j-1)+1\sim 64j$ in the space of the low-dimensional representation.
Thus, we set $\mcl{Z}\subseteq\mathbb{R}^{320}$.
For a new task $\mcl{T}_{\opn{new}}$, we set the finite-dimensional subspace $\mcl{V}$ of $\alg^N$ as $\oplus_{j=1}^{N_1}\opn{Span}\{v_{j,1},\ldots,v_{j,l}\}^{N_0}$.
Here, $l\le 10$, $v_{j,i}(z)=e^{-10\Vert z_{j,i}-p_j(z)\Vert^2}$, $z_{j,1}=Z(\mcl{T}_{\opn{new}})_{64(j-1)+1:64j}$, and $z_{j,i}$ for $i=2,\ldots,l$ are randomly drawn from the normal distribution with mean $z_{j,1}$ and standard deviation $0.01$.
Moreover, $p_j$ is the projection that maps $z$ to $z_{64(j-1)+1:64j}$ and for a finite-dimensional vector $v$, $v_{i:j}$ denotes the $j-i+1$-dimensional vector composed of elements $i\sim j$ of $v$.
We set the map ${P}:\alg\to\mcl{V}$ as the kernel ridge regression in the same manner as Section~\ref{subsec:nf_numexp}.

For randomly selected two new tasks, we compared the accuracy of the classification among different values of $l$.
We set $l=1,7,10$ for both tasks.
For $l>1$, the output (the probability to belong to each class) is a function on $\mcl{Z}$.
We computed the test accuracy of the output at $z_1$.
Fig.~\ref{fig:fs_accuracy} shows the results.
Note that $l=1$ corresponds to the standard few-shot learning (with $\mathbb{R}$-valued parameters) explained in~\ref{subsec:few_shot_form}.
We can get higher accuracy as $l$ becomes larger.
Fig.~\ref{fig:fs_param} shows the accuracy with different values of the hyper parameter $\eta$ of the kernel ridge regression $P$.
In this experiment, we fixed another hyperparameter $\lambda$ as $0$.
{If $\mu=0$, then our method is equivalent to the existing method. Thus, if $\mu$ is too small, the accuracy is not so high.
On the other hand, if $\mu$ is large, then the mean squared error of the kernel ridge regression becomes large.
Thus, if $\mu$ is too large, the accuracy is not so high.}

\begin{figure}[t]
\centering
 \subfigure[Task $1$]{\includegraphics[scale=0.27]{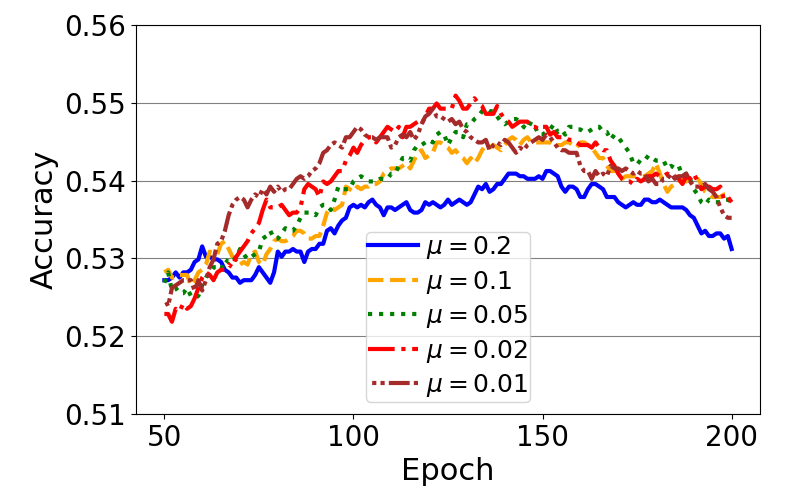}}\vspace{-.2cm}\\
\subfigure[Task $2$]{\includegraphics[scale=0.27]{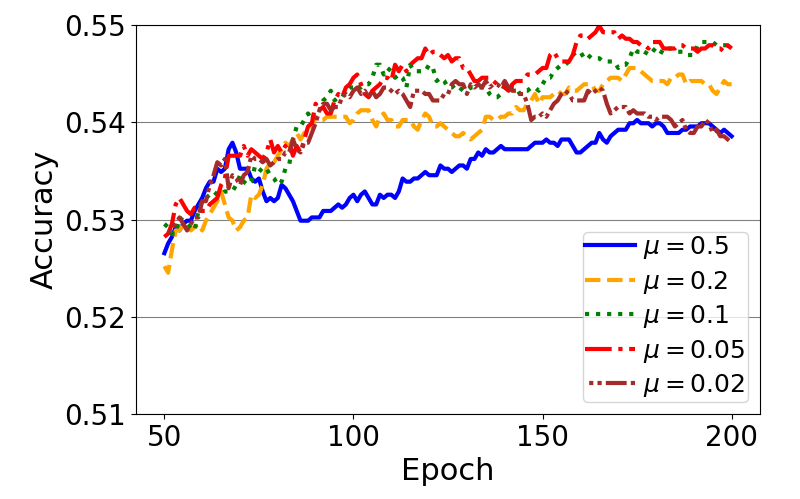}}\vspace{-.2cm}
\caption{{Test accuracies of the new tasks with different values of $\mu$. ($l=10$ for both tasks)}}\label{fig:fs_param}
\end{figure}

\subsection{Other applications}\label{subsec:other_appl}
Although in Subsections~\ref{subsec:density_estimation} and \ref{subsec:few-shot} we focus on density estimation and few-shot learning, we can apply our framework to other applications.
We list examples of other applications and discuss connections with existing methods below. 

\paragraph{Ensemble learning}
Ensemble learning combines multiple models to obtain better generalization performance~\cite{dong20,ganaie21}.
Our framework allows us to combine the models continuously.
Indeed, the case of $\tilde{\lambda}=\eta=0$ and $\mcl{D}=\sum_{i=1}^l\delta_{z_i}/l$ in Subsection~\ref{subsec:nf_numexp} is equivalent to the existing framework of ensemble learning.
In general, let $\mcl{Z}=\{z_1,\ldots,z_m\}$ be a finite discrete set and let $P$ and $\tilde{\lambda}$ in Eq.~\eqref{eq:gd_scheme_mod} be the identity map and $0$, respectively.
In addition, let the loss function $L$ be defined as $L(\theta)(z)=\tilde{L}(\theta(z),z)$. 
Then, our framework is reduced to the ensemble learning because for each $i=1,\ldots, m$, $f^{\theta}(\cdot)(z_i)$ in Eq.~\eqref{eq:nn_cstar} is the classical model with $\mathbb{R}$-valued parameters, and the learning process is independent of that for $j\neq i$. 
By setting $\mcl{Z}$ as an infinite set and $P$ as a map different from the identity map, we can learn multiple models more efficiently.


\paragraph{Generating time-series or spatial data}
We set $\mcl{Z}\subseteq\mathbb{R}$ for time-series data and we set $\mcl{Z}\subseteq\mathbb{R}^2$ or $\mcl{Z}\subseteq\mathbb{R}^3$ for spatial data.
Then, we set $f=f^{\theta}$ in Eq.~\eqref{eq:nn_cstar} so that $f^{\theta}(\cdot)(z)$ is a generative model such as GAN~\cite{goodfellow14,karras20}, VAE~\cite{kingma14,gregor15}, or normalizing flow~\cite{dinh14,dinh17,teshima20} for any $z\in\mcl{Z}$.
Since the outputs of the model are functions on $\mcl{Z}$ in our framework, we can generate time-series or spatial data as a function on $\mcl{Z}$ rather than a discrete series.

\paragraph{Learning distributions of parameters}
Distributions of parameters of models have been studied~\cite{pennington18,sonoda21}.
\citet{franchi20} propose learning distributions of parameters of a model rather than their values.
They assume that the distributions are normal distributions.
Using our framework, we can consider more general distributions.
Let $\mcl{Z}$ be a probability space and we set the $\alg$-valued parameters $\theta\in\alg^N$ as random variables taking their values in $\mathbb{R}$. 
If we limit the distributions of $\theta$ to Dirac measures, then our framework is reduced to be a classical model with $\mathbb{R}$-valued parameters.
If we limit the distributions of $\theta$ to normal distributions, then our framework is reduced to be the framework proposed by~\citet{franchi20}.

\paragraph{Generalizing complex-valued networks}
Using complex-valued variables and parameters of models for taking advantage of the arithmetic of complex numbers has been studied~\cite{bassey21,hirose92,amin08,nishikawa05,yadav05}.
Since $C^*$-algebra is a generalization of the space of complex numbers, our framework generalizes complex-valued networks.
In addition, the $C^*$-algebra $C(\mathcal{Z})$ is the space of {complex-valued} continuous functions on $\mathcal{Z}$.
Thus, by using our framework, we can aggregate multiple complex-valued networks in the same way as real-valued networks.
\color{black}

\section{Conclusion}\label{sec:conclusion}
In this paper, we proposed a new general framework of neural networks on $C^*$-algebra.
We focused on the $C^*$-algebra of the space of continuous functions on a compact space and provided a gradient descent method for learning the model on $C^*$-algebra.
By generalizing the parameters of a model to functions, we can use tools for functions such as regression and integrations, 
which enables us to learn features of data efficiently and adapt the models to problems continuously. 
We applied our framework to density estimation and few-shot learning and showed the validity of our framework.
Our framework is valid for a wide range of practical applications and not limited to the above cases.





\bibliography{example_paper}
\bibliographystyle{icml2022}


\newpage
\appendix
\onecolumn
\section*{Appendix}
\section{Definitions and examples related to Section~\ref{sec:background}}\label{ap:definition}
We provide the standard definitions and examples related to Section~\ref{sec:background}.
\begin{definition}[Algebra]\label{def:algebra}
A set $\alg$ is called an {\em algebra} on a field $\mathbb{F}$ if it is a vector space equipped with an operation $\cdot:\alg\times\alg\to\alg$ that satisfies the following conditions for $b,c,d\in\alg$ and $\alpha\in\mathbb{F}$:\vspace{.2cm}

 $\bullet$ $(b+c)\cdot d={b}\cdot d+c\cdot d$,\quad
 $\bullet$ $b\cdot(c+d)=b\cdot c+b\cdot d$,\qquad
 $\bullet$ $(\alpha c)\cdot d=\alpha(c\cdot d)=c\cdot(\alpha d)$.\vspace{.2cm}\\
\leftskip=0pt
The symbol $\cdot$ is omitted when doing so does not cause confusion.
\end{definition}
\begin{definition}[Multiplication]\label{def:multiplication}
Let $\modu$ be an abelian group with operation $+$ and let $\alg$ be a ring.
For $c,d\in\alg$ and $u,v\in\modu$, if an operation $\cdot:\modu\times\alg\to\modu$ satisfies

$\bullet$ $(u+v)\cdot c=u\cdot c+v\cdot c$,\qquad
$\bullet$ $u\cdot (c+d)=u\cdot c+u\cdot d$,\qquad
$\bullet$ $u\cdot (cd)=(u\cdot {c})\cdot {d}$,\qquad
$\bullet$ $u\cdot 1_{\alg}=u$ \textcolor{black}{if $\alg$ is unital},\vspace{.2cm}\\
then $\cdot$ is called an {\em $\alg$-multiplication}.
The symbol $\cdot$ is omitted when doing so does not cause confusion.
\end{definition}

\begin{example}\label{ex:An}
A simple example of Hilbert $C^*$ modules over a $C^*$-algebra $\alg$ is $\alg^N$ for $N\in\mathbb{N}$.
The $\alg$-valued inner product of $\mathbf{c}=[c_1,\ldots,c_n]^T$ and $\mathbf{d}=[d_1,\ldots,d_n]^T$ is defined as $\blacket{\mathbf{c},\mathbf{d}}=\sum_{i=1}^nc_i^*d_i$.
The norm in $\alg^N$ is given as $\Vert \mathbf{c}\Vert=\Vert \sum_{i=1}^nc_i^*c_i\Vert^{1/2}$.
\end{example}

\section{Proof of Fact~\ref{fact:uniform}}
\setcounter{section}{3}
\renewcommand{\thesection}{\arabic{section}}
\setcounter{theorem}{3}
\begin{fact}
Let $\mcl{Z}$ be a compact metric space and let $d_{\mcl{Z}}$ be the metric on $\mcl{Z}$.
Let $\zeta_0,\zeta_1,\ldots\in\alg$ be a sequence such that there exists a function $\zeta^*$ on $\mcl{Z}$, $\lim_{t\to\infty}\zeta_t(z)=\zeta^*(z)$ for any $z\in\mcl{Z}$.
If $\zeta_0,\zeta_1,\ldots$ is uniformly Lipschitz continuous, that is, there exists $C>0$ such that for any $t\in\mathbb{N}$, 
\begin{equation*}
\vert \zeta_t(z_1)-\zeta_t(z_2)\vert\le Cd_{\mcl{Z}}(z_1,z_2),
\end{equation*}
then $\zeta_0,\zeta_1,\ldots$ converges uniformly to $\zeta^*$ and $\zeta^*\in\alg$.
\end{fact}
\begin{proof}
 For $z,c\in\mcl{Z}$ and $s,t\in\mathbb{N}$, we have
\begin{align}
 \vert \zeta_s(z)-\zeta_t(z)\vert\le \vert \zeta_s(z)-\zeta_s(c)\vert + \vert \zeta_s(c)-\zeta_t(c)\vert + \vert \zeta_t(c)-\zeta_t(z)\vert
\le 2Cd_{\mcl{Z}}(z,c)+\vert \zeta_s(c)-\zeta_t(c)\vert,\label{eq:dist}
\end{align}
where the last inequality holds since $\zeta_0,\zeta_1,\ldots$ is uniformly Lipschitz continuous.
Let $\epsilon>0$.
Since $\zeta_t(c)$ converges to $\zeta^*(c)$, there exists $T_c>0$ such that for any $s,t\ge T_c$, $\vert \zeta_s(c)-\zeta_t(c)\vert\le\epsilon$ holds.
Let $\mcl{Z}_{c}=\{z\in\mcl{Z}\,\mid\,d_{\mcl{Z}}(z,c)<\epsilon\}$.
Then, we have $\mcl{Z}=\bigcup_{c\in\mcl{Z}}\mcl{Z}_c$.
Since $\mcl{Z}$ is compact, there exists $n\in\mathbb{N}$ and $c_1,\ldots,c_n\in\mcl{Z}$ such that $\mcl{Z}=\bigcup_{i=1}^n\mcl{Z}_{c_i}$.
For $z\in\mcl{Z}_{c_i}$, by Eq.~\eqref{eq:dist}, we have
\begin{align*}
 \vert \zeta_s(z)-\zeta_t(z)\vert\le 2C\epsilon+\epsilon
\end{align*}
for $s,t\ge T_{c_i}$.
Thus, for any $s,t\ge\max_{i\in\{1,\ldots,n\}}T_{c_i}$ and $z\in\mcl{Z}$, we have $\vert \zeta_s(z)-\zeta_t(z)\vert\le (2C+1)\epsilon$, which implies the uniform convergence of $\zeta_t$.
Since $\alg=C(\mcl{Z})$ is the Banach space equipped with the sup norm, $\zeta^*\in\alg$. 
\end{proof}

\begin{figure}[b]
\centering
 \subfigure[$\mathbb{R}$-valued model]{\includegraphics[scale=0.27]{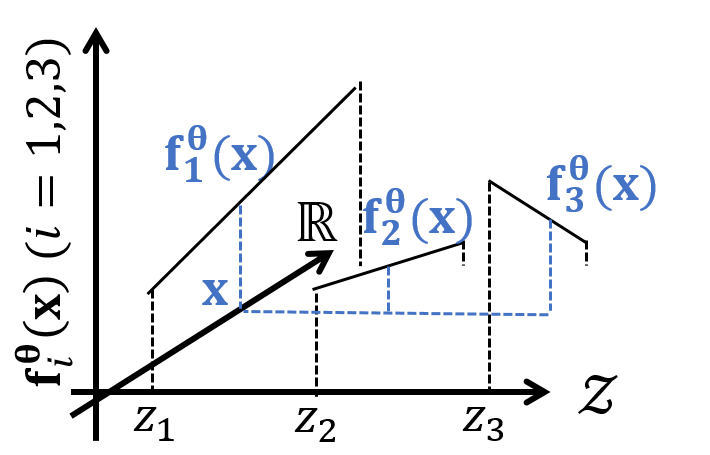}}\qquad
\subfigure[$C^*$-algebra-valued model]{\includegraphics[scale=0.27]{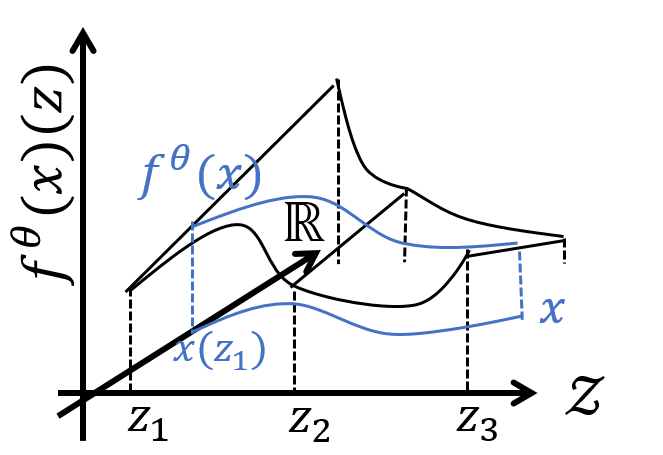}}
\caption{{Difference between $\mathbb{R}$-valued and $C^*$-algebra-valued models.}}\label{fig:linear_reg}
\end{figure}

\setcounter{section}{2}
\renewcommand{\thesection}{\Alph{section}}
\section{Difference between $\mathbb{R}$-valued and $C^*$-algebra-valued models}\label{ap:diff}
We illustrate the difference between $\mathbb{R}$-valued and $C^*$-algebra-valued models for a specific case of the linear regression problem on $\mathbb{R}$.
Let $z_1,z_2,z_3\in\mcl{Z}$ be variables that characterize models (e.g., the index of tasks and the class of data).
For $\mathbb{R}$-valued models, we separately learn three models and get $\mb{f}^{\bs{\theta}}_i(\mb{x})=\bs{\theta}_{i,1}\mb{x}+\bs{\theta}_{i,2}$ for $i=1,2,3$ and $\mb{x}\in\mathbb{R}$.
Here, $\bs{\theta}_{i,j}\in\mathbb{R}$ is an $\mathbb{R}$-valued parameter.
On the other hand, for $C^*$-algebra-valued models, we aggregate three models continuously and represents them as a function.
Then, we learn the model and get $f^{\theta}(x)(z)=\theta_{i,1}(z)x(z)+\theta_{i,2}(z)$ for $i=1,2,3$ and $x\in\alg$, which corresponds to learn the models simultaneously with interactions.

\end{document}